\documentclass[letterpaper, 10 pt, conference]{ieeeconf}  

\IEEEoverridecommandlockouts                              

\usepackage{amsmath}
\usepackage{amsfonts}
\usepackage{mathtools}
\usepackage{amssymb}
\usepackage{amsthm}
\usepackage{graphicx}        
\usepackage{float}
\usepackage{hyperref}
\usepackage{multirow}
\usepackage{tabularx}
\usepackage{booktabs}
\usepackage{hhline}
\usepackage{url}
\usepackage{subfig}
\usepackage{wrapfig}
\usepackage[font={small}]{caption}

\newtheorem{theorem}{Theorem}
\newtheorem{definition}{Definition}

\newtheorem{proposition}{Proposition}

\title{\LARGE \bf
Learning Autonomous Vehicle Safety Concepts from Demonstrations}

\author{Karen Leung$^{\star\dagger}$, Sushant Veer$^{\star}$, Edward Schmerling$^{\star\ddagger}$, Marco Pavone$^{\star\ddagger}$
\thanks{$^\star$NVIDIA, $^\dagger$University of Washington, $^\ddagger$Stanford University, {\tt\small \{kymleung@uw.edu, kaleung@nvidia.com\}}}%
}

\newcommand{\jointstate}{x}

\newcommand{\vA}{v_\mathrm{A}}         
\newcommand{\vB}{v_\mathrm{B}}         

\newcommand{\uA}{u_\mathrm{A}}    
\newcommand{\uB}{u_\mathrm{B}}    

\newcommand{\UA}{\mathcal{U}^\mathrm{A}}    
\newcommand{\UB}{\mathcal{U}^\mathrm{B}}    
\newcommand{\UAz}{\widetilde{\mathcal{U}}^\mathrm{A}}    
\newcommand{\UBz}{\widetilde{\mathcal{U}}^\mathrm{B}}    

\newcommand{\dataset}{\mathcal{S}}

\newcommand{\hjvalue}{\mathcal{V}}
\newcommand{\targetset}{\mathcal{T}}

\newcommand{\hocbf}{\mathrm{HOCBF}}

\newcommand{\lie}[2]{\mathcal{L}_{#1}#2}
\newcommand{\liem}[3]{\mathcal{L}_{#1}^{#3}#2}

\begin{document}

\maketitle
\thispagestyle{empty}
\pagestyle{empty}

\begin{abstract}
Evaluating the safety of an autonomous vehicle (AV) depends on the behavior of surrounding agents which can be heavily influenced by factors such as environmental context and informally-defined driving etiquette.
A key challenge is in determining a minimum set of assumptions on what constitutes \textit{reasonable foreseeable behaviors} of other road users for the development of AV safety models and techniques.
In this paper, we propose a \textit{data-driven} AV safety design methodology that first learns ``reasonable'' behavioral assumptions from data, and then synthesizes an AV safety concept using these learned behavioral assumptions.
We borrow techniques from control theory, namely high order control barrier functions and Hamilton-Jacobi reachability, to provide inductive bias to aid interpretability, verifiability, and tractability of our approach. In our experiments, we learn an AV safety concept using demonstrations collected from a highway traffic-weaving scenario, compare our learned concept to existing baselines, and showcase its efficacy in evaluating real-world driving logs.
\end{abstract}

\section{Introduction}
\label{sec:introduction}

As autonomous vehicle (AV) operations grow, developing appropriate methods for evaluating AV safety becomes ever more imperative. 
The question of ``is a vehicle in an unsafe state?'' is relevant for AV system developers, policymakers, and the general public alike (see Figure~\ref{fig:hero}). While \emph{guaranteeing} safety may not be practical in the face of the myriad uncertainties and complexities that come with real-world driving, there is still a broad desire to codify, to some extent, collectively agreed-upon notions of safety \cite{IEEEP28462022}.
Should safety be defined using data-driven methods that can account for the complexities of the AV's environment but lack interpretability and formal guarantees, or leverage control theoretic techniques derived from first principles which are interpretable and rigorous but not as scalable or expressive as their learned counterparts? 
In this work, we strike a middle-ground by a learning safety-critical driving behavior model and integrating it within a robust control framework to develop an interpretable and rigorous AV safety model that is informed by data.

Towards this goal of building safe and trustworthy AVs, various stakeholders have advanced \textit{safety concepts} consisting, in general, of two functions mapping world state (e.g., joint state of all agents and environmental context) to (i) a scalar measure of safety, and (ii) a set of allowable (safe) agent actions \cite{LeungBajcsyEtAl2022}.
Numerous uses of such safety concepts have been proposed throughout AV pipelines, e.g., a criterion to prune away unsafe plans \cite{PhanMinhHowingtonEtAl2022}, a safety monitor to determine when evasive action must be taken \cite{LeungSchmerlingEtAl2020}, a component in the planning objective \cite{WangLeungEtAl2020}, or for perception safety evaluation metrics \cite{OborilBuerkleEtAl2022,TopanLeungEtAl2022}.
Critically, different behavioral assumptions lead to different safety concepts which in turn affects realized AV safety and performance.
To mitigate overconservatism and enable maximum flexibility, we contend that the design of an effective safety concept hinges upon characterizing reasonable foreseeable behaviors of other agents, in a way conducive to efficiently evaluating the safety of driving scenes and producing associated safe controls.

To address this challenge, we propose designing novel safety concepts by learning from data what are controls, specifically, control sets, that humans operate with when their safety is threatened, and then using these learned control sets to inform AVs of what are reasonable foreseeable behaviors of other agents in safety-critical scenarios. Equipped with such a learned ``human behavior collision avoidance model,'' we perform safety concept synthesis by using robust control theory, specifically Hamilton-Jacobi reachability \cite{MitchellBayenEtAl2005}, as a powerful inductive bias for interpretability, verifiability, and tractability. 
Our control set learning approach differs from reward learning paradigms (i.e., inverse reinforcement learning / inverse optimal control \cite{AbbeelNg2004,ZiebartMaasEtAl2008,LevineKoltun2012}) which strive to learn high-level human intentions from demonstrations, often in the absence of constraints \cite{PhanMinhHowingtonEtAl2022,SadighSastryEtAl2016c}. Reward learning approaches aim to model a human's high-level planning objective which encapsulates more than just safety considerations. Whereas in this work, we are interested in \textit{learning control constraint boundaries associated with collision avoidance behaviors} as opposed to nuanced behaviors (e.g., aggressive versus passive).

\begin{figure}[t]
    \centering
    \includegraphics[width=0.47\textwidth]{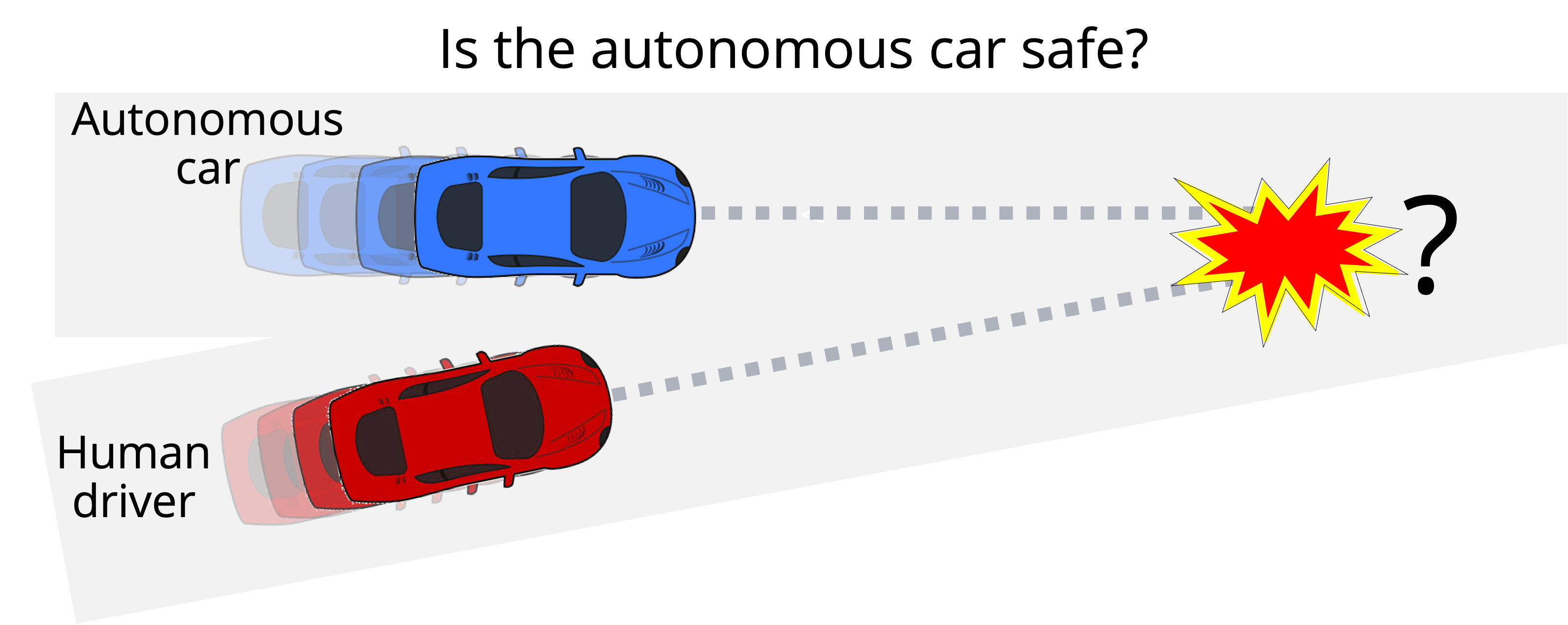}
    \caption{Evaluating the safety of an autonomous vehicle (AV) depends on what constitutes as \textit{reasonable foreseeable behaviors} of other road users. For example, determining whether the autonomous car (blue) is currently in a safe state depends on how the human driver (red) may behave (e.g., speed up or swerve away). In this work, we develop a data-driven methodology to model reasonable foreseeable behaviors and leverage the learned model for AV safety evaluation.}
    \label{fig:hero}
\end{figure}

\noindent \textbf{Structure and Contributions.}
We provide a literature review in Section~\ref{sec:related work}, give an overview on Hamilton-Jacobi (HJ) reachability in Section~\ref{sec:hj reachability}, and formally state our safety concept learning problem in Section~\ref{sec:problem formulation}. Then we describe the details of our key contributions:
\textbf{(i)} We propose a data-driven approach to learn humans' collision avoidance behaviors in the control space to capture ``reasonable driving behaviors'' (Section~\ref{sec:learning control sets}). Specifically, we learn safe control sets from demonstrations via a high order control barrier function (HOCBF) \cite{XiaoBelta2021} framework. 
\textbf{(ii)} We develop a constrained game-theoretic optimization problem derived from a HJ reachability formulation to synthesize novel data-driven safety concepts that are robust to other agents' behaviors while respecting the learned collision avoidance behaviors (Section~\ref{sec:data-driven safety}).
\textbf{(iii)} We demonstrate our proposed learning framework using highway driving data and show that the resulting data-driven safety concept is less conservative than other common safety concepts (due to the way it captures constraints on reasonable agent behavior) and thus is useful as a ``responsibility-aware'' evaluation metric for the safety of AV interactions (Section~\ref{sec:results}).


\section{Related Work}
\label{sec:related work}

We use the term \textit{safety concept} to help unify existing safety theory prevalent in various robot planning and control algorithms, such as velocity obstacles \cite{FioriniShiller1998,WilkieVanDenBergEtAl2009}, forward reachable sets \cite{HolmesKousikEtAl2020,AlthoffDolan2014}, contingency planning \cite{KuwataTeoEtAl2009,JansonHuEtAl2018}, backward reachability \cite{MitchellBayenEtAl2005}, and other methods that make static assumptions on agent behavior \cite{NisterLeeEtAl2019,ShalevShwartzShammahEtAl2017}. The differences between various safety concepts stem from the assumptions about the behavior of other interacting agents, ranging from worst-case assumptions \cite{MitchellBayenEtAl2005} to presuming agents follow fixed open-loop trajectories (e.g., braking \cite{ShalevShwartzShammahEtAl2017}, constant velocity \cite{WilkieVanDenBergEtAl2009}).
Indeed, a core challenge is in selecting behavioral assumptions that balances conservatism, tractability, interpretability, and compatibility with real-world interactions.

Recent works propose dynamically changing the conservatism of the safety concept based on online estimations of how confident the robot's human behavior prediction model is. If a human agent is behaving as expected (i.e., high model confidence), then the worst-case assumptions in the safety concept can be relaxed, and vice versa \cite{FridovichKeilBajcsyEtAl2019,TianSunEtAl2022}. 
However, the integrity of the adaptive safety concept depends on the quality of the prediction model where obtaining an accurate human behavior prediction model is in general quite challenging and, indeed, is an active research field \cite{RudenkoPalmieriEtAl2020}.

Another data-driven safe control technique is to use expert demonstrations and learn a control barrier function (CBF) \cite{AmesCooganEtAl2019} to describe unsafe regions in the state space \cite{RobeyHuEtAl2020,RobeyLindemannEtAl2021,QinZhangEtAl2021,LyuLuoEtAl2022}. The learned CBF is then \textit{directly} used as the core safety mechanism in synthesizing a safe policy. However, CBFs are not well-suited for interactive settings where there is uncertainty in how other interacting agents may behave.
In our work, we too consider learning CBFs (specifically, high order CBFs (HOCBFs) \cite{XiaoBelta2021}), but instead use the learned HOCBF as an intermediate step towards formulating a more rigorous notion of safety rooted in robust control theory which enjoys interpretability and verifiability benefits.

\section{Safety Concept via Hamilton-Jacobi Reachability}
\label{sec:hj reachability}

We define a safety concept as a combination of two functions mapping world state to (i) a scalar measure of safety, and (ii) a set of allowable actions for each agent in which to preserve safety. A family of safety concepts can be described via a HJ reachability formulation \cite{LeungBajcsyEtAl2022}.

HJ reachability is a mathematical formalism used for characterizing the safety properties of dynamical systems \cite{MitchellBayenEtAl2005,MargellosLygeros2011}. The outputs of a HJ reachability computation are (i) a HJ value function, a scalar-valued function that measures ``distance'' to collision, and (ii) a set of controls that prevents the safety measure from decreasing further---precisely the components needed for a safety concept.

Consider a target set $\targetset$ which is the set of collision states between agents $A$ (ego agent) and $B$ (contender).
The HJ reachability formulation describes a two-player differential game to determine whether it is possible for the ego agent to avoid entering $\targetset$ under any family of closed-loop policies of the contender, as well as the ego agent's appropriate control policy for ensuring safety. 
It is assumed that the contender follows an \textit{adversarial} policy and has the advantage with respect to the information pattern.
Using the principle of dynamic programming, the collision avoidance problem reduces to solving the Hamilton-Jacobi-Isaacs (HJI) partial differential equation (PDE)\cite{MitchellBayenEtAl2005},

\vspace{-3mm}
{\small
\begin{equation}
    \begin{aligned}
    &\frac{\partial \hjvalue(\jointstate, t)}{\partial t} + 
    \min \Big\{0,~
    {\begingroup
    \max_{\uA \in \UA} \min_{\uB \in  \UB}
    \endgroup}
    \nabla_{\jointstate} \hjvalue(\jointstate, t)^\top f(\jointstate, \uA, \uB) \Big\} = 0\\
    &\hjvalue(\jointstate, 0) = \ell(\jointstate)
\end{aligned}
\label{eq:hji pde}
\end{equation}
}

\noindent where $\jointstate \in \mathcal{X}$ denotes the joint state of agents $A$ and $B$, $\uA \in \UA$ and $\uB\in\UB$ are the available (bounded) controls\footnote{The control sets $\UA$ and $\UB$ are typically chosen to reflect the physically feasible limits of the system.} of agents $A$ and $B$, respectively, and $f(\cdot,\cdot,\cdot)$ is the joint dynamics assumed to be measurable in $\uA$ and $\uB$ for each $\jointstate$, and uniformly continuous, bounded, and Lipschitz continuous in $\jointstate$ for fixed $\uA$ and $\uB$.\footnote{This assumption ensures that trajectories are generated by a unique control sequence.} 
The boundary condition is defined by a function $\ell: \mathcal{X} \rightarrow \mathbb{R}$ whose zero sub-level set encodes the target set, i.e., $\targetset =  \lbrace \jointstate \mid \ell(\jointstate) < 0\rbrace$.
The solution $\hjvalue(\jointstate,t),\, t\in [-T,0]$, called the HJ value function, captures the lowest value of $\ell(\cdot)$ along the system trajectory within $|t|$ seconds if the systems starts at $\jointstate$ and both agents $A$ and $B$ act optimally, that is, $\uA^*(\jointstate),\,\uB^*(\jointstate) = \arg \max_{\uA \in \UA}(\arg)\min_{\uB \in \UB} \nabla_{\jointstate} \hjvalue(\jointstate, t)^\top f(\jointstate, \uA, \uB)$. Thus the HJ value function fulfills the first aspect of a safety concept. 
After obtaining the HJ value function, we can also consider the set of controls that prevent the HJ value function from decreasing over time. That is, we can compute the \textit{safety-preserving control set},  $\mathcal{U}^\mathrm{A}_{\mathrm{safe}}(\jointstate) = \lbrace \uA \in \UA \mid  \min_{\uB \in \UB} \frac{d\hjvalue(\jointstate,t)}{dt} \geq 0 \rbrace $, thus fulfilling the second aspect of a safety concept.
By varying the problem parameters, i.e., control sets, behavior type, and a choice of $\ell(\cdot)$, we can synthesize a \textit{family} of safety concepts via HJ reachability. 

Although solving HJI PDE suffers from the curse of dimensionality, it is performed \emph{offline}. For reasonably sized problems, such as the pairwise car-car system considered in the work, solving the HJI PDE is tractable on modern computers. Equipped with the HJ value function and a state $x$, computing $\hjvalue(\jointstate, t)$ and $\mathcal{U}^\mathrm{A}_{\mathrm{safe}}(\jointstate)$ is computationally lightweight, consisting of a table look-up, interpolation and/or matrix-vector multiplication.


\section{Towards Data-driven Safety Concept Synthesis}
\label{sec:problem formulation}
A key limitation to the standard HJ reachability setup is in the assumption the contender will choose \textit{optimal collision-seeking} controls from $\UB$, the set of admissible controls of the system (typically corresponding to the system's physical limits). While a worst-case assumption is key in providing robustness guarantees for the safety of the system, the assumption that a contender agent will execute \textit{any} available, including worst-case, controls is unrealistic in practice.

To reduce the conservatism in how we model a contender's behavior within the HJ formulation (and hence resulting safety concept), we aim to use data to learn \textit{state-dependent control sets} to capture agent behaviors reflecting practical real-world operations. That is, \textbf{the problem we seek to solve consists of two parts: (i) learning state-dependent control sets from data, and (ii) integration of the state-dependent control sets into the HJ reachability formulation for safety concept synthesis.}
Building upon the notation introduced in Section~\ref{sec:hj reachability}, consider a nonlinear control-disturbance-affine\footnote{Many vehicle dynamics models can be written as a control-affine system.} system describing the joint dynamics between agent $A$ and $B$,
\begin{equation}
    \dot{\jointstate} = f(\jointstate) + g(\jointstate)\uA + h(\jointstate)\uB
    \label{eq:pairwise dynamics}
\end{equation}
where $\jointstate \in \mathcal{X} \subset \mathbb{R}^n$,  $\uA \in \UA\subset \mathbb{R}^{m_\mathrm{A}}$ and $\uB\in\UB \subset \mathbb{R}^{m_\mathrm{B}}$, and $f: \mathcal{X} \rightarrow \mathbb{R}^n$, $g: \mathcal{X} \rightarrow \mathbb{R}^{n \times m_\mathrm{A}}$, and $h: \mathcal{X} \rightarrow \mathbb{R}^{n\times m_\mathrm{B}}$ are locally Lipschitz continuous. 

\noindent{\bf Learning state-dependent control sets.} Consider an offline dataset $\dataset = \{ (\jointstate^{(i)}, \uA^{(i)}) \}_{i=1}^{N}$ of joint states between agent $A$ and $B$, and agent $A$'s controls collected from \textit{safe} interactions between agents $A$ and $B$. We assume agent $A$ does not observe the controls of agent $B$.
We seek to find a mapping $\phi: \mathcal{X} \rightarrow 2^{\UA}$ such that $\uA^{(i)} \in \phi(x^{(i)})$ for all $(\jointstate^{(i)}, \uA^{(i)}) \in \dataset$. Of course $\phi$ can be a trivial mapping, i.e., $\phi(x^{(i)}) = \UA$ for all $\jointstate \in \mathcal{X}$; thus we additionally require the learned state-dependent sets to be \textit{minimal} in the sense that the set should contain the relevant points as tightly as possible.

\noindent{\bf Safety concept synthesis.} Equipped with $\phi$, we seek to integrate the learned control sets into the HJ reachability framework to \textit{restrict} the feasible control set that an agent will operate with.


\section{Learning State-Dependent Control Sets \\From Data}
\label{sec:learning control sets}
In this section, we describe our proposed approach to describing human collision avoidance behaviors by learning state-dependent control sets from data. 
First we give a self-contained overview to HOCBFs (Section~\ref{subsec:hocbf}), then followed by our proposed learning procedure (Section~\ref{subsec:hocbf learning}).

\subsection{High-Order Control Barrier Functions}
\label{subsec:hocbf}

HOCBF provides a method to represent and reason about control invariant sets \cite[Theorem 4]{XiaoBelta2021}. 
Before introducing HOCBFs, we first introduce control barrier functions (CBFs), a special case of HOCBFs.
Consider a control-affine system of the form,
\begin{equation}
    \dot{x}=f(x) + g(x)u, \: \text{where}\: x\in\mathcal{X}\subset \mathbb{R}^n, \: u\in \mathcal{U} \subset \mathbb{R}^m,
    \label{eq:control affine dynamics}
\end{equation}
where $f: \mathcal{X} \rightarrow \mathbb{R}^n$ and $g: \mathcal{X} \rightarrow \mathbb{R}^{m_\mathrm{A}}$ are locally Lipschitz continuous. 
Suppose there is a set which the system wants to stay inside of (e.g., set of safe states).
A control invariant set is a set of states for which it is possible for a system \eqref{eq:control affine dynamics} to stay inside of indefinitely. In many settings, it is often desirable to find the largest control invariant set that is contained within the safe set.

\begin{definition}[Control invariant set]
A set $\mathcal{C}$ is control invariant for a system $\dot{x}=f(x,u)$ if for all $x_0\in \mathcal{C}$, there exists a control law $u=k(x) \in \mathcal{U}$ such that $x(t) \in \mathcal{C},\, x(0) = x_0,\, \dot{x}(t)=f(x(t),u)$ for all $t\in[0,\infty)$.
\label{def:control invariant set}
\end{definition}

\noindent One way to represent a control variant set is to use Control Barrier Functions, a scalar valued function which satisfies a set of certain conditions which we describe next.
\begin{definition}[Extended Class $\mathcal{K}_\infty$ function]
A continuous function $\alpha :(-\infty, \infty) \rightarrow (-\infty, \infty)$, belongs to extended class $\mathcal{K}_\infty$ if it is strictly increasing, $\alpha(0) = 0$, and $\lim_{r\rightarrow \pm\infty} \alpha(r) = \pm\infty$.
\end{definition}

\begin{definition}[Control Barrier Function  \cite{AmesCooganEtAl2019}]
Given a system \eqref{eq:control affine dynamics}, consider a continuously differentiable function $b: \mathcal{X} \rightarrow \mathbb{R}$ and a set $\mathcal{C} \subset \mathcal{X}$ where $\mathcal{C} = \{ x \in \mathcal{X} \mid b(x) \geq 0\}$. Then $b$ is a candidate control barrier function (CBF) if there exists a class $\mathcal{K}_\infty$ function $\alpha$ such that
\begin{equation}
    \sup_{u\in \mathcal{U}} \left[ \lie{f}{b}(x) + \lie{g}{b}(x)u \right] \geq - \alpha(b(x)) \quad  \forall x\in\mathcal{X},
    \label{eq:cbf}
\end{equation}
where $\lie{f}{b}(x)=\nabla b(x)^Tf(x)$ and $\lie{g}{b}(x)=\nabla b(x)^Tg(x)$ are Lie derivatives describing the change in $b(x)$ along a vector field described by $f(x)$ and $g(x)$, respectively.
\label{def:cbf}
\end{definition}

\noindent A valid CBF ensures that there exists a feasible control that prevents $b(x(t))$ from decreasing faster than a rate of $-\alpha(b(x))$ along system trajectories.
If $\lie{g}{b}(x) = 0$, the control input $u$ has no influence in ensuring that \eqref{eq:cbf} holds. In fact, if the \textit{relative degree} $m_\mathrm{r}$ of the system and $b$ is greater than one, then $\lie{g}{b}(x) = 0$ for all $x\in\mathcal{X}$.

\begin{definition}[Relative degree]
The relative degree $m_\mathrm{r}$ of a (sufficiently) differentiable function $b: \mathbb{R}^n \rightarrow \mathbb{R}$ with respect to dynamics \eqref{eq:control affine dynamics} is when $\lie{g}{\liem{f}{b}{m_\mathrm{r}-1}}(\cdot) \neq 0$ but $\lie{g}{\liem{f}{b}{m_\mathrm{r}-2}}(\cdot) = 0$.
\end{definition}

\noindent When $m_\mathrm{r} > 1$, the standard CBF formulation (Definition~\ref{def:cbf}) cannot be used to determine what are safe control actions. Thus we consider HOCBFs \cite{XiaoBelta2021} instead which addresses this issue, and we show later that HOCBFs are a more natural paradigm to encode obstacle sets.

\begin{definition}[High-order Control Barrier Function  \cite{XiaoBelta2021}]
Given a system \eqref{eq:control affine dynamics}, and an $m^{\mathrm{th}}$ order differentiable function $b: \mathcal{X} \rightarrow \mathbb{R}$. Let $m_\mathrm{r}\leq m$ be the relative degree of $b$ with respect to the dynamics. Define a sequence of functions $\psi_i(x) = \dot{\psi}_{i-1}(x) + \alpha_i(\psi_{i-1}(x)),\, i=1,...,m_\mathrm{r}, \, \, \psi_0(x) = b(x)$.

Then $b$ is a candidate high order control barrier function (HOCBF) of relative degree $m_\mathrm{r}$ for system \eqref{eq:control affine dynamics} if there exists differentiable extended class $\mathcal{K}$ functions $\alpha_i,\,i=1,...,m_\mathrm{r}$ such that

{\footnotesize
\begin{equation}
    \sup_{u\in \mathcal{U}}  \underbrace{\left[ \liem{f}{b}{m_\mathrm{r}}(x) + \lie{g}{\liem{f}{b}{m_\mathrm{r}-1}}(x)u + \mathcal{O}(b(x)) +\alpha_{m_\mathrm{r}}(\psi_{{m_\mathrm{r}}-1}(x))\right]}_{\psi_{m_\mathrm{r}}(x)}  \geq 0, 
    \label{eq:hocbf}
\end{equation}
}
\noindent where $\mathcal{O}(b(x)) = \sum_{i=1}^{{m_\mathrm{r}}-1} \liem{f}{(\alpha_{{m_\mathrm{r}}-i} \circ \psi_{{m_\mathrm{r}}-i-1}}{i})(x)$.
\label{def:hocbf}
\end{definition}

\begin{theorem}[\cite{XiaoBelta2021}, Theorem 4]
Given a valid HOCBF $b(x)$ for a system \eqref{eq:control affine dynamics} and sets $\mathcal{C}_i,\, i=1,...,m_\mathrm{r}$, $\mathcal{C}_i = \{ x \in \mathcal{X} \mid \psi_{i-1}(x) \geq 0 \}$. Then any Lipschitz controller $u\in K_\hocbf(x)$ renders the set $\cap_{i}^m \mathcal{C}_i$ forward invariant, where
\begin{eqnarray}
\small
K_\hocbf(x) = \{u\in \mathcal{U} \mid \liem{f}{b}{m_\mathrm{r}}(x) + \lie{g}{\liem{f}{b}{m_\mathrm{r}-1}}(x)u + \notag\\
\mathcal{O}(b(x)) +\alpha_{m_\mathrm{r}}(\psi_{{m_\mathrm{r}}-1}(x)) \geq 0\}.
\label{eq:hocbf control set}
\end{eqnarray}
\normalsize
\end{theorem}

\begin{definition}[Effective CBF]
Given a HOCBF $b$ and corresponding $\alpha_i,\, i=1,...,m_\mathrm{r}$, we say $\psi_{m_\mathrm{r}-1}$ is the ``effective'' CBF since \eqref{eq:hocbf} is equivalent to \eqref{eq:cbf} with $b(x) = \psi_{m_{\mathrm{rel}-1}}(x)$.
\label{def:effective cbf}
\end{definition}

In short, given system dynamics and an unsafe set, HOCBFs can describe control invariant sets. The size and shape of a control invariant set is influenced by the $\alpha_i$ functions which imposes a bound on the allowable rate at which the system is allowed to approach the unsafe set.

\subsection{State-dependent Control Set Learning Procedure}
\label{subsec:hocbf learning}
In this section, we detail our state-dependent control set learning procedure. Roughly, we use HOCBF as inductive bias in learning a mapping from state to control set. 

\subsubsection{Learning set-up and assumptions} We make the assumption that (safe and reasonable) drivers do not put themselves in situations where it is not possible for them to avoid a collision. That is, drivers execute controls that ensure they continually stay inside a ``safe set'' because leaving that set may lead to collision. For example, (typical) safe drivers avoid approaching obstacles at high speeds because they may not be able to steer and slow down to avoid the obstacle. This so-called ``safe-set'' is precisely a control invariant set (see Definition~\ref{def:control invariant set}). 
Thus, given a dataset of \textit{safe} collision-free trajectories, the goal is to learn a control invariant set and corresponding control sets that contain the observed samples. We represent the control invariant set using HOCBFs \cite[Theorem 4]{XiaoBelta2021} which provides inductive bias that help make the learning process tractable and a direct and interpretable way to obtain state-dependent control sets (see \eqref{eq:hocbf control set}) amenable for safety concept synthesis.

For collision avoidance, agents are aiming to avoid the \textit{same} obstacle set but only differ in \textit{how} they avoid it. 
Thus we use HOCBFs to avoid reasoning about obstacle sets in velocity states (or higher derivatives) which is difficult to construct, and instead reason about $\alpha_i$'s.

\subsubsection{Illustrative example}
\begin{figure}[t]
    \centering
    \includegraphics[width=0.5\textwidth]{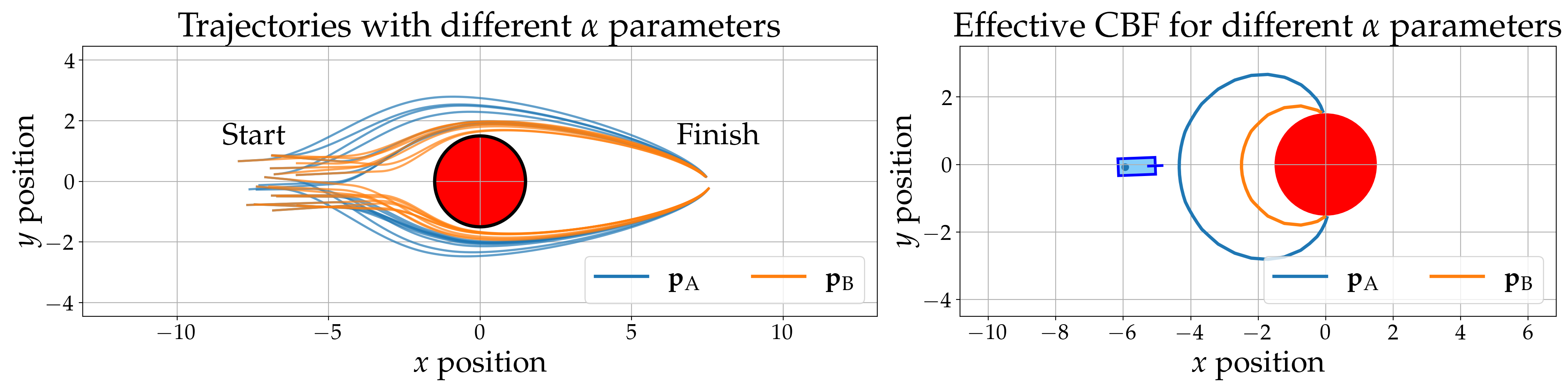}
    \caption{Two different $\alpha_i$ functions parameterized by different parameters, $\mathbf{p}_\mathrm{A}$ and $\mathbf{p}_\mathrm{B}$, result in different trajectory behaviors and effective CBFs. Left: More ``aggressive'' behaviors are observed with using $\mathbf{p}_\mathrm{B}$ compared to $\mathbf{p}_\mathrm{A}$. Right: The ground truth effective CBF each set of trajectories.}
    \label{fig:hocbf illustration}
\end{figure}
Figure~\ref{fig:hocbf illustration} illustrates how different parameters for the HOCBF $\alpha_i$ functions result in different collision avoidance behaviors, thereby motivating our goal of learning HOCBFs, specifically parameters for $\alpha_i$, from data.
Suppose a car needs to avoid a circular obstacle centered at $x_\mathrm{ob}$ with radius $r_\mathrm{ob}$ while reaching a goal state $x_\mathrm{goal}$. We use a HOCBF $b(x) = \|x - x_\mathrm{ob}\|^2_2 - r_\mathrm{ob}^2$ to represent the obstacle (i.e., the unsafe set). 
Figure~\ref{fig:hocbf illustration} illustrates two sets of trajectories generated from the HOCBF-CLF planner described in \cite{XiaoBelta2021}, each set constructed using different parameters, $\mathbf{p}_\mathrm{A}$ and $\mathbf{p}_\mathrm{B}$ for $\alpha_i$. We see that $\mathbf{p}_\mathrm{A}$ (blue) results in more conservative behaviors than $\mathbf{p}_\mathrm{B}$ (orange) since the blue trajectories starts turning away from the obstacle earlier. Correspondingly, we see the zero sub-level set of the effective CBF with $\mathbf{p}_\mathrm{A}$ is larger than the one for $\mathbf{p}_\mathrm{B}$.

\subsubsection{HOCBF learning algorithm} Next, we present our HOCBF learning algorithm whereby we aim to learn the parameters for $\alpha_i$ from data. Let $\dataset = \{ (x^{(i)}, u^{(i)}) \}_{i=1}^{N}$ be a dataset containing states and controls collected from humans operating in the target environment exhibiting \textit{safe} collision avoidance behaviors.
Since negative examples are rare and their explicit collection is impractical, one of the key strengths of our method is its ability to learn control invariant sets from just positive examples. A;though our method can also consider negative examples if available.

Suppose we have dynamics given by \eqref{eq:control affine dynamics}. Let $\mathcal{C}$ describe the set of obstacle states in position space (we consider $\mathbb{R}^2$ but this analysis can extend to $\mathbb{R}^3$). Let $b: \mathcal{X} \rightarrow \mathbb{R}$ be a function where $b(x) \leq 0 \Leftrightarrow \mathrm{pos}(x) \in \mathcal{C}$, and $\mathrm{pos}:\mathcal{X} \rightarrow \mathbb{R}^2$ is a function that extracts the position states. Let $m_\mathrm{r}\geq1$ be the relative degree of \eqref{eq:control affine dynamics} and $b$.
We seek to find $\alpha_i, \, i=1,...,m_\mathrm{r}$ such that $u^{(j)} \in K_\hocbf(x^{(j)})$ for all $(x^{(j)}, u^{(j)}) \in \dataset$.
More concretely, let $\alpha_i^{p_i}$ indicate that $\alpha_i$ is parameterized by a vector of parameters $p_i$, and let $\mathbf{p} = [p_1,...,p_{m_\mathrm{r}}]$. 
Considering that (i) $\alpha_i^{p_i}$ can be chosen to be arbitrarily large (i.e., steep) which will trivially lead to \eqref{eq:hocbf control set} being true for all datapoints, (ii) we want the HOCBF control constraint \eqref{eq:hocbf control set} to be satisfied, (iii) the states should be inside the (safe) control invariant set, and (iv) the dataset could contain outliers or noise, we seek to find $\mathbf{p}^\star$ such that,
\begin{equation}
\begin{split}
    \mathbf{p}^\star =& \arg\min_{\mathbf{p}} \frac{1}{N}\sum_{j=1}^N\underbrace{-\beta_1\min\{G_x^ju^{(j)} + F_x^j(\mathbf{p}), 0\}}^{\text{\eqref{eq:hocbf control set} violation}} + \\
    &\qquad \qquad \qquad \beta_2 \underbrace{\max\{\tanh(G_x^ju^{(j)} + F_x^j(\mathbf{p})), 0\}}^{\text{\eqref{eq:hocbf control set} satisfaction (saturated)}} + \\
    & \qquad \qquad \qquad \beta_3( \underbrace{-\min\{ \psi_{m_\mathrm{r}-1}(x^{(j)};\mathbf{p}), 0\}}_{\text{Effective CBF violation}} +\\
    &\qquad \qquad \qquad \beta_4 \underbrace{\max\{ \tanh(\psi_{m_\mathrm{r}-1}(x^{(j)};\mathbf{p})), 0\}}_{\text{Effective CBF satisfaction (saturated)}} + \\
    & \qquad \qquad \qquad \beta_5 \underbrace{\|\mathbf{p}\|_2^2}_{\text{$\alpha_i$ regularization}}
    \label{eq:hocbf learning}
\end{split}
\end{equation}
where $G_x^j = \lie{g}{\liem{f}{b}{m_\mathrm{r}-1}}(x^{(j)})$, $F_x^j(\mathbf{p}) =\liem{f}{b}{m_\mathrm{r}}(x^{(j)})+ \mathcal{O}(b(x^{(j)});\mathbf{p}) +\alpha_{m_\mathrm{r}}^{p_{m_\mathrm{r}}}(\psi_{{m_\mathrm{r}}-1}(x^{(j)}))$, and $\mathcal{O}(b(x);\mathbf{p}) = \sum_{i=1}^{{m_\mathrm{r}}-1} \liem{f}{(\alpha_{{m_\mathrm{r}}-i}^{p_{m_\mathrm{r}}-i} \circ \psi_{{m_\mathrm{r}}-i-1}}{i})(x)$ (c.f., \eqref{eq:hocbf control set}). 
We highlight some important considerations for this optimization problem.

\noindent{\it Saturation:} We apply $\tanh$ on the second and fourth terms in \eqref{eq:hocbf learning} because while we desire all the data points to satisfy the HOCBF conditions, we are not too interested in \textit{how much} each data point satisfies them. That is, we want to reduce the influence of extremely safe points (e.g., states very far and moving away from the obstacle).

\noindent{\it Parameterization:} The $\alpha_i, \, i=1,...,m_\mathrm{r}$ functions must belong to extended class $\mathcal{K}_\infty$. As such, we can parameterize $\alpha_i$ as a linear combination of extended class $\mathcal{K}_\infty$ basis functions. In future work, we can consider using a monotonic neural network \cite{Sill1997,WehenkelLouppe2019}.

\noindent{\it Pairwise joint dynamics:} In the case with pairwise joint dynamics, as presented in \eqref{eq:pairwise dynamics}, there is an additional term corresponding to the contender's input (i.e., a disturbance term), and ultimately, it will show up in \eqref{eq:hocbf learning}. 
In particular, the HOCBF constraint becomes
\begin{equation}
\begin{split}
    &\underbrace{\lie{g}{\liem{f}{b}{m_\mathrm{r}-1}}(\jointstate)\uA}_{\text{Ego's influence}} + \underbrace{\lie{h}{\liem{f}{b}{m_\mathrm{r}-1}}(\jointstate)\uB}_{\text{Contender's influence}} + \\
    & \qquad \liem{f}{b}{m_\mathrm{r}}(\jointstate)+ \mathcal{O}(b(\jointstate);\mathbf{p}) +\alpha_{m_\mathrm{r}}^{p_{m_\mathrm{r}}}(\psi_{{m_\mathrm{r}}-1}(\jointstate)) \geq 0
\end{split}
    \label{eq:hocbf control disturbance constraint}
\end{equation}
which is affine in $\uA$ and $\uB$.
To select a value for $\uB$, depending on what information is available, we could either (i) use the ground truth value, (ii) assume the worst-case, or (iii) \textit{predict} a distribution over agent $B$'s controls and take the worst-case with respect to the predictions.

\noindent{\it Differentiability:} If the dynamics and $b$ are differentiable, then \eqref{eq:hocbf learning} can be optimized via standard gradient descent algorithms. In this work, we use JAX \cite{JAX2018}, a Python software library for automatic differentiation.

\noindent{\it Negative examples:} If negative examples are available (e.g., from real-world driving logs or synthetically generated \cite{ChenDathathriEtAl2019}), we can include additional terms to \eqref{eq:hocbf learning} that encourage the negative examples to violate the control constraint and have negative effective CBF values. The negative examples will help refine the delineation of the control invariant set.


\section{Safety Concept Synthesis via Constrained HJ Reachability}
\label{sec:data-driven safety}

We incorporate the learned HOCBF into the HJ reachability formulation to synthesize a safety concept that assumes \textit{reasonable} human behaviors in safety-critical situations. 
Directly adding the learned HOCBF control constraint \eqref{eq:hocbf control disturbance constraint} to \eqref{eq:hji pde} places the burden of constraint satisfaction on the agent that acts second (i.e., agent $B$), thus imposing a strong assumption that the other (uncontrolled) agents will always behave responsibly with respect to HOCBF constraint satisfaction. To address this, we swap the ordering in which the agents act, and the following proposition states that this swapping results in a more conservative HJ value function, i.e., it is more advantageous for a player to act first.

\begin{proposition}
Consider a coupled affine constraint $p(\uA, \uB) := a\uA + b\uB + c \geq 0$ in $\uA$ and $\uB$, and a linear objective $q(\uA, \uB)=A\uA + B\uB$. Let $\UAz = \{ \uA\in\UA \mid \exists \uB\in\UB, \, p(\uA,\uB) \geq 0\}$ and $\UBz = \{ \uB\in\UB \mid \exists \uA\in\UA, p(\uA,\uB) \geq 0\}$ represent control sets for each agent which ensures the other agent can satisfy the constraint $p(\uA, \uB) \geq 0$. Let $\UA(\uB) = \{ \uA\in\UA \mid p(\uA,\uB) \geq 0\}$, and $\UB(\uA) = \{ \uB\in\UB \mid p(\uA,\uB) \geq 0\}$ describe an agent's feasible control set with the other agent's control fixed. Then, 
\begin{equation}
    \max_{\uA \in \UAz} \min_{\uB \in \UB(\uA)} q(\uA, \uB) \geq  \min_{\uB \in \UBz} \max_{\uA \in \UA(\uB)} q(\uA, \uB).
    \label{eq:minimax constrained inequality}
\end{equation}
\label{prop:minimax constrained inequality}
That is, the player that acts first has the advantage assuming the second player is provided feasible options.
\end{proposition}

\begin{proof}
\begin{eqnarray*}
&&\max_{\uA \in \UAz} \min_{\uB \in \UB(\uA)} q(\uA, \uB) \\
&\geq& \max_{\uA \in \UAz} \min_{\uB \in \UBz} q(\uA, \uB) \quad {\small(\UB(\uA) \subseteq \UBz)}\\
&=& \min_{\uB \in \UBz} \max_{\uA \in \UAz}  q(\uA, \uB) \quad {\small(\text{Minimax theorem})}\\
&\geq& \min_{\uB \in \UBz} \max_{\uA \in \UA(\uB)}  q(\uA, \uB)  \quad {\small (\UA(\uB) \subseteq \UAz)}
\end{eqnarray*}
\end{proof}

Proposition~\ref{prop:minimax constrained inequality} is contrary to the well-known minimax inequality for the unconstrained case where the player acting second has the advantage \cite{Neumann1927}. 
Since we assume control-disturbance-affine dynamics, both the HOCBF constraint from \eqref{eq:hocbf control set} and $\nabla \hjvalue(\jointstate,t)^Tf(\jointstate, \uA, \uB)$ term from \eqref{eq:hji pde} are linear in control input, we can apply Proposition~\ref{prop:minimax constrained inequality} and hence instead consider a \textit{constrained} HJ reachability problem,
\vspace{2mm}
\begin{equation}
\begin{split}
    &\textbf{Constrained HJ reachability}\\
    &\min_{\uB \in  \UBz}  \max_{\uA \in \UA(\uB)}\nabla_{\jointstate} \hjvalue(\jointstate, t)^\top \left[ f(\jointstate) + g(\jointstate)\uA + h(\jointstate)\uB\right]\\
    &\quad  \text{such that} \: \eqref{eq:hocbf control disturbance constraint} \:\text{holds}.
   \label{eq:hji with cbf constraint}
\end{split}
\end{equation}
Instead of assuming the contender will follow a worst-case policy (the standard HJ reachability setup), \eqref{eq:hji with cbf constraint} computes the worst-case \textit{responsible} contender policy, where we consider responsible behaviors as those that ensure HOCBF control constraint satisfaction.
This constrained formulation assumes that contender (agent $B$) will act first and choose a control that ensures there exists a feasible control for the ego (agent $A$) to choose from. Thus the ego bears the burden of constraint satisfaction. Future work includes investigating ways to divide collision avoidance responsibility.


\section{Experiments and Discussion}
\label{sec:results}
We present two examples, a simple static obstacle collision avoidance example to aid understanding of our method, and a highway driving scenario studied in \cite{SchmerlingLeungEtAl2018} to illustrate the synthesis of a data-driven safety concept for AV applications. We note that our method can be applied to other settings, such as car-pedestrian interactions, and social navigation. We plan on releasing our code upon publication.

\subsection{Collision avoidance with circular obstacle}

\begin{figure}[t]
     \centering
     \subfloat[Effective CBFs.]{\includegraphics[width=0.16\textwidth]{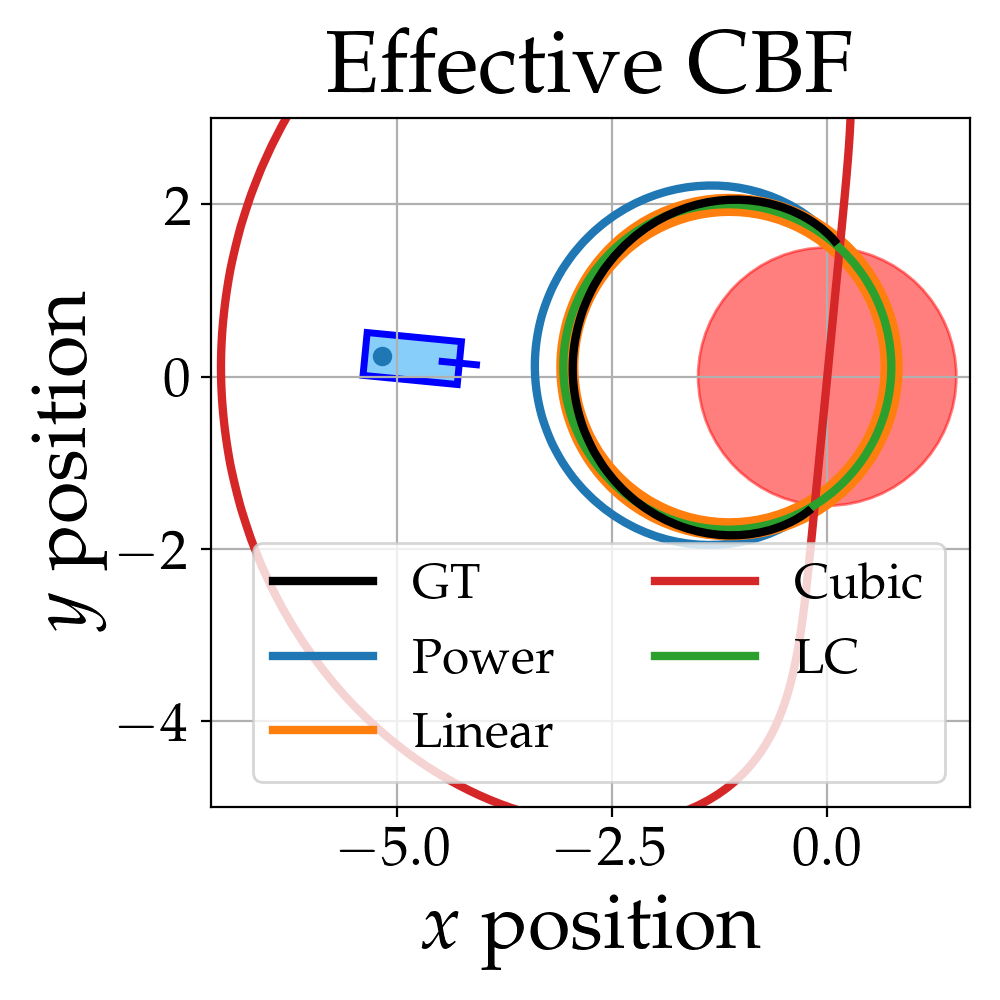}\label{fig:toy alpha parameterization comparison}}
     \subfloat[Synthesized  trajectories.]{\includegraphics[width=0.32\textwidth]{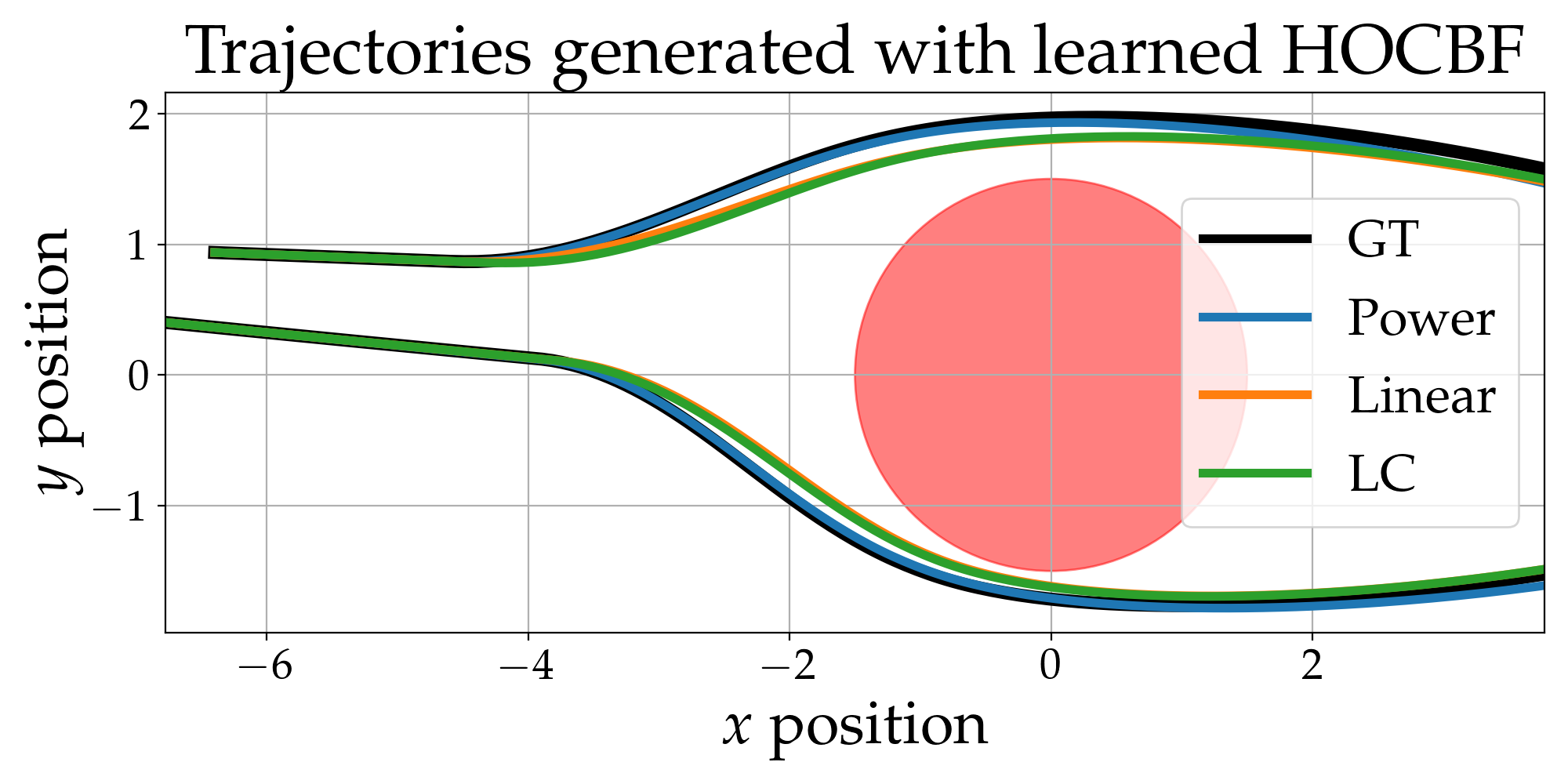}\label{fig:toy generated traj from learned hocbf}}
     \caption{Comparison of various learned HOCBFs for collision avoidance around a circular obstacle.}
     \label{fig:toy example results}
\end{figure}
We continue with the example discussed in Section~\ref{subsec:hocbf learning}.
Consider a 4-state simple car model,
\begin{equation}
    [\dot{x}, \dot{y}, \dot{\theta}, \dot{v}] = [v\cos\theta , v\sin\theta , \frac{v}{\ell}\tan\delta , a]
    \label{eq:simple car}
\end{equation}
with $(\delta,a)$ as steering and acceleration control inputs, and $\ell$ is the wheelbase length of the car. 
For the HOCBF learning, we use the following parametric form $\alpha_i^{p_i}(a) = p_{i,1}a^{p_{i,2}}$, $p_1 = [0.54, 1.16],\, p_2=[0.68, 1.11]$ to generate a set of trajectories with various random initial conditions. The parameter values serves as the ground truth (GT).
We applied the HOCBF learning procedure using various parameterization of $\alpha_i^{p_i}$: power function (same parameterization as GT), linear function $\alpha_i^{p_i} = p_ia$, cubic function $\alpha_i^{p_i}=p_ia^3$, and a linear combination (LC) of a linear, cubic, and $\tanh$ function $\alpha_i^{p_i}(x) = p_{i,1}x + p_{i,2}\tanh(p_{i,3}x) + p_{i,4}x^3$ where $p_i=[p_{i,1}, p_{i,2}, p_{i,3}, p_{i,4}]$, all with $i=1,2$. 

When learning the parameters of $\alpha$ (i.e., performing gradient descent on \eqref{eq:hocbf learning}), we used $\beta_1 = 1$, $\beta_2 = \beta_4 = 0.001$, $\beta_3 = 1$, $\beta_5 =0.001$, step size / learning rate = $0.001$, and 150000 gradient steps (though it converged much earlier).
Figure~\ref{fig:toy example results} visualizes the zero level set of the learned effective CBF (left) and resulting trajectory from running the same reach-avoid planner used to generate the dataset but with the different learned $\alpha_i$ parameterizations (right). We see that the learned HOCBFs except the cubic parameterization is able to closely recover the ground truth behavior (black line).

\subsection{Highway driving}
In this example, we consider a highway driving scenario, synthesize a novel safety concept based on a highway driving dataset, and evaluate on real traffic data.

\subsubsection{Data collection} 
\begin{figure}[t]
    \centering
    \includegraphics[width=0.48\textwidth]{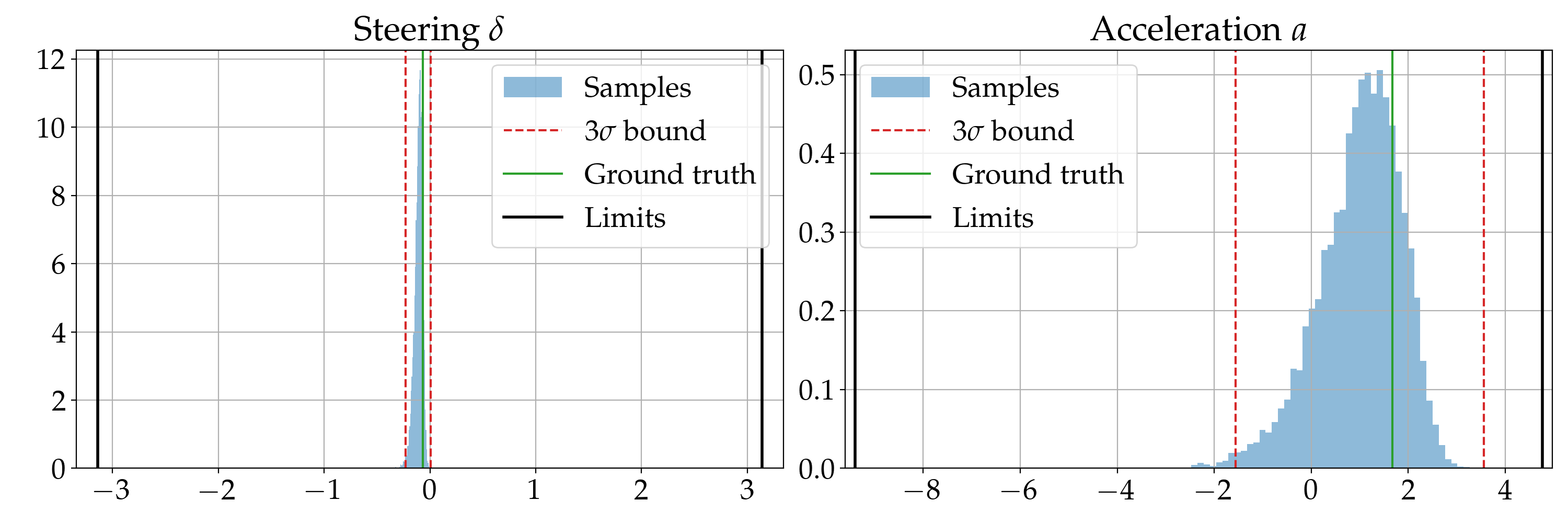}
    \caption{Comparison between the ground truth control (green), distribution of predicted control (blue), three standard deviation bounds of the prediction distribution (red), and upper/lower control limits (black).}
    \label{fig:predicted_3std_example}
    \vspace{-3mm}
\end{figure}
We use the traffic-weaving dataset \cite{SchmerlingLeungEtAl2018}, and assume the control inputs of the other cars are not observable (this is often true in real-world settings). Instead, we learn a generative model to construct distributions over likely contender controls (i.e., disturbances) and use them as ground truth for our HOCBF learning procedure. 
We trained a conditional variational autoencoder with a continuous latent space (latent dimension size of 8) with a history length of ten time steps, and prediction horizon of five. Although we predict five time steps into the future, we only consider the first (i.e., current) control of the other agent. We used LSTM cells for the encoder and decoder with hidden dimension 8. Given the prediction, we define $\mathcal{D}_\mathrm{pred}$ as the interval three standard deviations away from the mean, and then compute $\uB \in \mathcal{D}_\mathrm{pred}$ that minimizes the LHS of \eqref{eq:hocbf control disturbance constraint} during the HOCBF learning process. Figure~\ref{fig:predicted_3std_example} shows an example which compares the learned distribution, three standard deviation bound, the ground truth, and control limits. 

\begin{table}[t]
    \centering
    \vspace{2mm}
    \caption{Values of a hyperparameter sweep. The values chosen for the rest of the experiments are shown in \textbf{bold}.}
    \begin{tabular}{|c|c|c|}
    \hline
        \textbf{Parameter}  & \textbf{Values} \\ \hline \hline
        $\beta_1$ & $\lbrace \mathbf{1} \rbrace$ \\ 
        $\beta_2$ & $\lbrace 0.0, \mathbf{0.001}, 0.01, 0.1 \rbrace$\\
        $\beta_3$ &  $\lbrace 0., 0.1, \mathbf{1.0} \rbrace$\\
        $\beta_4$ & (=$\beta_2$)\\
        $\beta_5$ & $\lbrace 0.1, \textbf{0.001} \rbrace $ \\
        $\alpha_i$ & $\lbrace \mathbf{p_1 x},\, p_1 x + p_2\tanh(p_4 x) + p_3 x^3\rbrace $\\
        Learning rate & $\lbrace \mathbf{0.001} \rbrace $\\
        Number of steps & $\lbrace \mathbf{10000} \rbrace$\\\hline
    \end{tabular}
    \label{tab:hyperparameter sweep}
\end{table}

\begin{figure}[t]
     \centering
     \subfloat[Zero level-set of HJ value function evaluated at $\Delta x-\Delta y$ slice. $\vA=18$ms$^{-1}$, $\vB=25$ms$^{-1}$, $\theta_\mathrm{A}=\theta_\mathrm{B}=0$rad. ]{\includegraphics[width=0.47\textwidth]{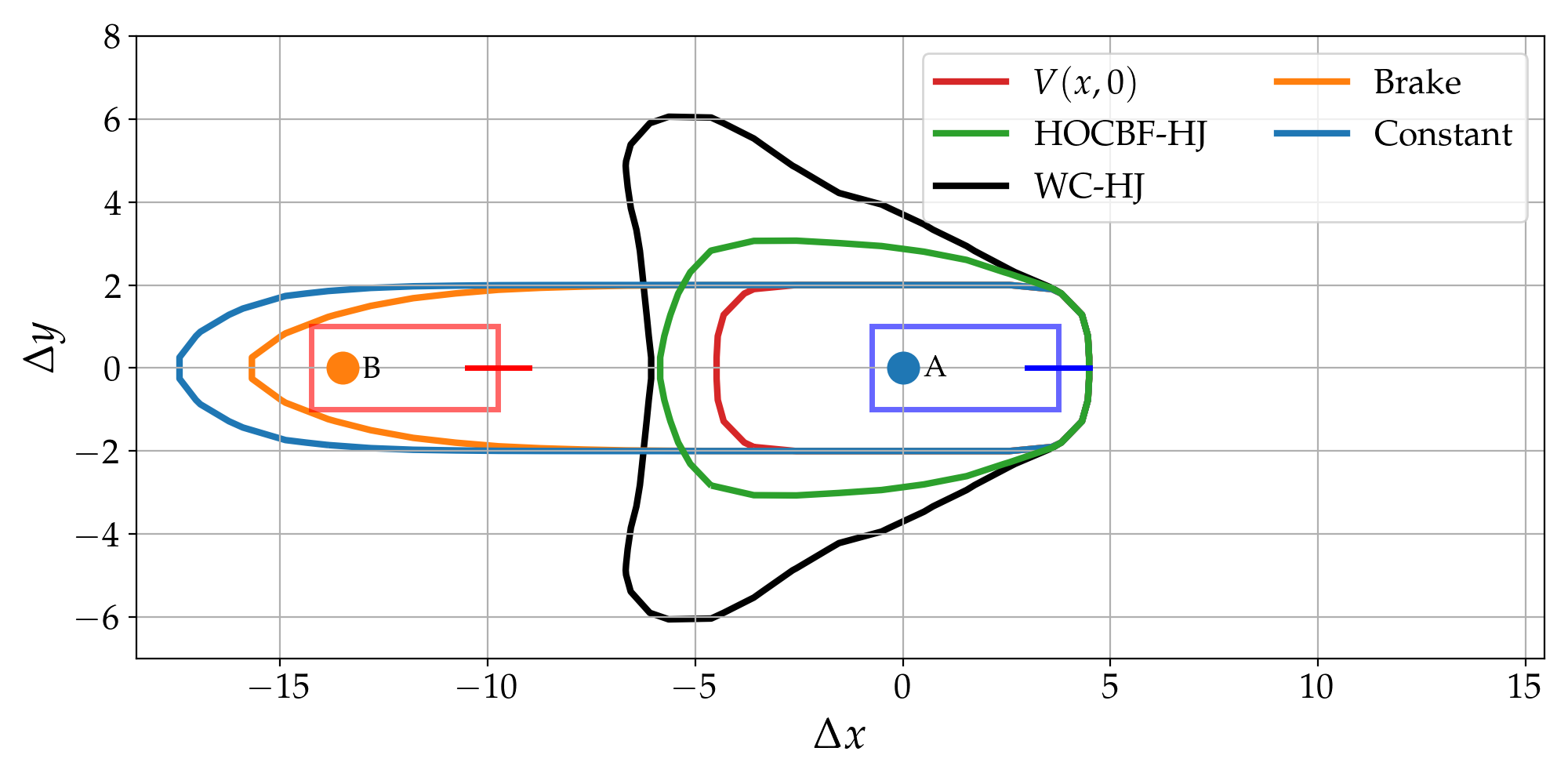}\label{fig:trafficweaving avoid set constrained unconstrained comparison, b behind}} \hspace{3mm}
     \subfloat[HOCBF-HJ and WC-HJ optimal controls for agents $A$ and $B$ when agent $B$ is approaching from behind.]{\includegraphics[width=0.47\textwidth]{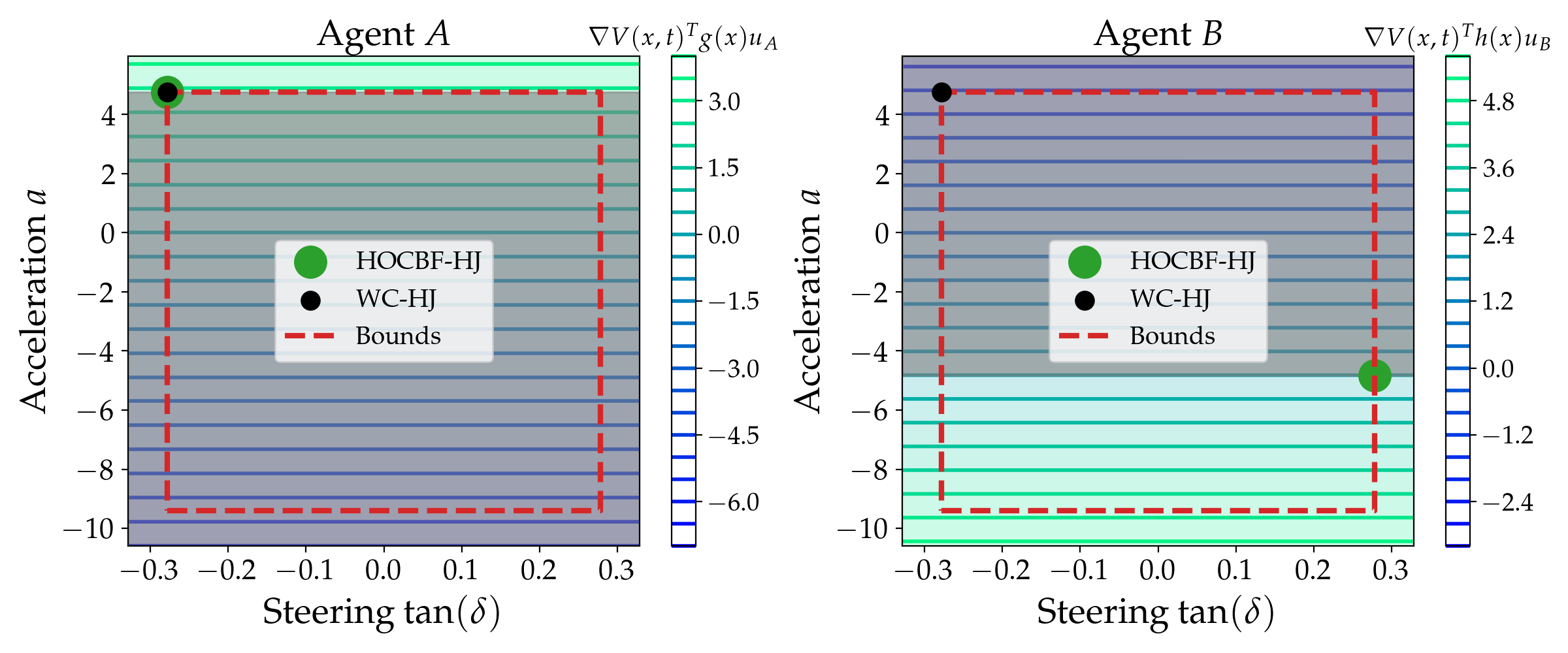}\label{fig:trafficweaving control comparison, b behind}}
     \caption{Safety concept comparison computed over a two second horizon. Left: Level sets of the HJ value function corresponding to different safety concepts. Right: Plots of control sets with infeasible regions according to the HOCBF constraint shaded gray.}
     \label{fig:trafficweaving HOCBF-HJ results}
     \vspace{-4mm}
\end{figure}

\subsubsection{Safety concept analysis} 
\begin{table*}[t]
\centering
\caption{Volumetric ``confusion matrices'' compared against the unconstrained (standard) HJ value function. Values represent the percentage of grid points in state space.}
\label{tab:avoid set volume comparison}
\begin{tabular}{cc|cc|cc|cc|cc}
\toprule
{} & {} & \multicolumn{2}{c|}{\bf WC-HJ} &\multicolumn{2}{c|}{\bf HOCBF-HJ} & \multicolumn{2}{c|}{\bf Brake} & \multicolumn{2}{c}{\bf Constant}\\
{} & {} & Safe & Unsafe & Safe & Unsafe & Safe & Unsafe & Safe & Unsafe \\ \midrule
\multirow{2}{*}{\bf WC-HJ} & Safe & 89.08 & 0 & 88.93 & 0.15 &  86.56 & 2.52 & 86.28  & 2.79\\
& Unsafe & 0 & 10.92 & 1.30 & 9.62   &  1.52 & 9.40   & 1.49  & 9.43\\
\bottomrule
\end{tabular}
\end{table*}

\begin{table}[t]
    \centering
    \caption{HJ value statistics corresponding to various safety concepts evaluated on the NGSIM dataset.}
    \label{tab:trafficweaving ngsim percentile}
    \begin{tabular}{|c|c|ccccc|}
    \hline
        \multirow{2}{*}{\bf Safety Concept} & \multirow{2}{*}{\bf Mean} & \multicolumn{5}{c|}{\bf Percentile}\\
        &  & {\bf 0$^\mathrm{th}$} & {\bf 5$^\mathrm{th}$} & {\bf 50$^\mathrm{th}$} & {\bf 95$^\mathrm{th}$} & {\bf 100$^\mathrm{th}$}  \\ \hline\hline
         {\bf WC-HJ}    & 4.3  & -0.38 & 0.24 & 4.04 & 9.36  & 13.67\\
         {\bf HOCBF-HJ} & 4.42 & -0.37 & 0.23 & 4.17 & 9.54  & 13.62 \\
         {\bf Brake}    & 4.96 &  0.13 & 0.65 & 4.52 & 10.94 & 14.93 \\
         {\bf Constant} & 4.52 & -0.01 & 0.55 & 4.21 & 9.52  & 13.83\\ \hline
    \end{tabular}
    \vspace{-3mm}
\end{table}

We consider a HOCBF $b(\jointstate) =  \frac{\Delta x^2}{a^2}  + \frac{\Delta y^2}{b^2} - 1$, $a = 5.4,\, b = 2.4$, and consider relative dynamics between two vehicles modeled by the simple car model \eqref{eq:simple car}, resulting in a relative degree of two.
We applied the learning procedure outlined in Section~\ref{subsec:hocbf learning}; we performed a hyperparameter sweep over the following parameters defined in \eqref{eq:hocbf learning}. The values are presented in Table~\ref{tab:hyperparameter sweep}, and the bolded values are the hyperparameters that we found yielded the best results.

With our learned HOCBF, we then solved the corresponding HOCBF constrained HJ reachability problem described by \eqref{eq:hji with cbf constraint}. As our problem is relatively low-dimensional, we can solve the constrained HJ reachability problem exactly by enumerating over the vertices of the feasible domain (but adds significant computational costs). We consider a two second horizon when solving the constrained HOCBF-HJ reachability problem and compare against various safety concepts (also synthesized via a HJ computation).

Figure~\ref{fig:trafficweaving HOCBF-HJ results} shows the unsafe region and optimal controls for various safety concepts: \textbf{HOCBF-HJ} is the proposed constrained HJ set-up, \textbf{WC-HJ} is the standard HJ formulation using worst-case contender policy without any control set restriction, \textbf{Brake} assumes both agents perform a maximum braking maneuver with zero steering (inspired by safety concepts used by AV companies \cite{NisterLeeEtAl2019,ShalevShwartzShammahEtAl2017}), and \textbf{Constant} assumes both agents maintain their current velocity \cite{WilkieVanDenBergEtAl2009}, a common assumption used in AV risk assessment methods \cite{DahlRodriguesdeCamposEtAl2019}. 
Qualitatively, we observe in Figure~\ref{fig:trafficweaving avoid set constrained unconstrained comparison, b behind} that due to the restrictions imposed by the HOCBF constraints, the HOCBF-HJ unsafe set can be significantly less ``wide'' compared to WC-HJ. The skinniness of the HOCBF-HJ unsafe set reflects practical driving assumptions where we expect nearby drivers to stay within their lanes. Figure~\ref{fig:trafficweaving avoid set constrained unconstrained comparison, b behind} also highlights the over-conservatism and impracticality of WC-HJ which would render cars traveling alongside agent $A$ in adjacent lanes as unsafe. A volumetric comparison\footnote{We considered regions where the vehicles speeds were $[15, 30$]ms$^{-1}$ and headings between $[-0.4\pi, 0.4\pi]$.} is provided in Table~\ref{tab:avoid set volume comparison} where we compare the size of the safe/unsafe regions. 
We see that compared to the Brake and Constant safety concepts, the HOCBF-HJ unsafe set is the most contained within the WC-HJ unsafe set---the regions which WC-HJ deems unsafe but HOCBF-HJ / Brake / Constant deems safe are of a similar ballpark (between 1.30\% to 1.49\%), but the size of the region which WC-HJ deems safe and HOCBF-HJ unsafe is significantly lower compared to Brake and Constant (0.15\% versus 2.52\% and 2.79\% respectively).
Figure~\ref{fig:trafficweaving control comparison, b behind} illustrates the optimal control for a car-following scenario depicted in Figure~\ref{fig:trafficweaving avoid set constrained unconstrained comparison, b behind}. The gray shaded region denotes the infeasible controls determined by the HOCBF constraint, and we see that the optimal HOCBF-HJ control (disturbance), shown by the green dot, maximizes (minimizes) $\nabla \hjvalue(x,t)^Tg(x)\uA$ ($\nabla \hjvalue(x,t)^Th(x)\uB$) while satisfying the HOCBF constraint. While the optimal control and disturbance for WC-HJ, denoted by the black dot, are the extremes of the control bounds. Figure~\ref{fig:trafficweaving control comparison, b behind} reflects intuitive and expected driving behaviors---it is unreasonable for agent A to assume that agent B will aggressively accelerate from behind.

\subsubsection{Offline evaluation on real traffic data:} We evaluate how various safety concepts perform \textit{ex post facto} on the NGSIM dataset \cite{NGSIM2016}, a highway driving dataset containing vehicles traveling along the US 101 highway in Los Angeles, USA. There are no collisions observed, and therefore we expect all the datapoints to be considered ``safe''.
Table~\ref{tab:trafficweaving ngsim percentile} presents the mean and various percentile values of the HJ value of each safety concept. While it is difficult to definitively say that a particular safety concept is better than another based on these statistics since there are a lot of nuances in the interpreting the HJ value, we make the following observations: (i) The Brake safety concept presents the highest values in all the statistics, indicating that it is the most optimistic out of the safety concepts studied. Being overly-optimistic can lead to over-confident driving where the AV can over-estimate how safe a situation is. (ii) For situations when the HJ value is low (i.e., for low percentiles), both the HOCBF-HJ and WC-HJ safety concept present very similar statistics, but the HOCBF-HJ statistics gradually become more similar to the Constant statistics as we consider higher percentile values. This suggests that (perhaps not too surprisingly) HOCBF-HJ behaves like a hybrid between the WC-HJ and Constant safety concepts.

\section{Limitations, Future Work, and Conclusions}
\label{sec:limitation_future_conclusions}

We conclude by highlighting some limitations of this work, laying out exciting future directions, and summarizing our key contributions.

\subsection{Limitations}
A major limitation of this work stems from (i) the curse of dimensionality in solving the HJI PDE explicitly (as opposed to implicit representations \cite{BansalTomlin2021}), and (ii) multi-agent extensions to the HOCBF and HJ reachability theory. Although we are currently using pairwise decomposition when considering multi-agent scenarios, pairwise decomposition for multi-agent analysis is effective in practice \cite{WangLeungEtAl2020,NisterLeeEtAl2019}, and in some settings, multi-agent analysis reduces to a pairwise decomposition \cite{ChenSingletaryEtAl2021}. While the eventual goal is to investigate using implicit representation (e..g, neural networks) when solving the HJI PDE to account for high dimensional scenarios, the (HO)CBF and HJ reachability formulation nonetheless provides useful inductive biases in building responsibility-aware safety concepts.

Our HOCBF-constrained safety concepts does run the risk of encountering behaviors not abiding by the HOCBF constraint and therefore invalidating assurances provided by the synthesized safety concept. While we can further improve the behavior learning step to help reduce model mismatch, there has been complementary work in out-of-distribution (OOD) detection to catch anomalous events \cite{FaridVeerEtAl2021}. Such OOD detectors are designed to be used as a complementary safety filter and catch rare tail events that are not nominally considered by the system.

\subsection{Future work} 
An area primed for future work is in investigating the division of responsibility for the constraint satisfaction of \eqref{eq:hji with cbf constraint}. A similar consideration is studied in \cite{GuoWangEtAl2021} which analyzes how much each agent should contribute to constraint satisfaction in a multi-agent setting. Fruitful future work entails learning appropriate responsibility allocation within a game theoretic context \cite{TsaknakisHongEtAl2021,GoktasGreenwald2021}, and combining offline safety concept synthesis with online responsibility learning.

\subsection{Conclusions} We have proposed a data-driven safety concept that is robust yet reflective of real-world driving interaction behaviors. We first learn control sets describing collision avoidance behaviors by leveraging high order control barrier functions, and then use them to constrain the HJ reachability computation used for safety concept synthesis. We show that our learned safety concept is less overly-conservative than other common AV safety concepts, and captures responsible behaviors, though there is more work to be done in testing safety concepts in closed-loop AV operations, and extensions to multi-agent interactive scenarios, especially where the division of responsibility for collision avoidance is ambiguous.

\bibliographystyle{ieeetran}
\bibliography{../../../bib/main,../../../bib/ctrl_papers}  

\end{document}